\documentclass[12pt]{article}
\pdfoutput=1
\usepackage{algorithm}
\usepackage{algpseudocode}
\usepackage{array}
\usepackage[caption=false,font=normalsize,labelfont=sf,textfont=sf]{subfig}
\usepackage{textcomp}
\usepackage{stfloats}
\usepackage{url}
\usepackage[english]{babel}
\usepackage{graphicx}
\usepackage{framed}
\usepackage[normalem]{ulem}
\usepackage{amsmath}
\usepackage{amsthm}
\usepackage{amssymb}
\usepackage{mathtools}
\usepackage{amsfonts}
\usepackage{xfrac}
\usepackage{enumerate}
\usepackage{comment}
\usepackage[utf8]{inputenc}
\usepackage[top=1 in,bottom=1in, left=1 in, right=1 in]{geometry}

\usepackage{hyperref} 
\usepackage{cleveref}
\crefname{subsection}{subsection}{subsections}
\crefname{subsection}{subsection}{subsections}

\usepackage{verbatim} 
\usepackage{comment}

\usepackage{xcolor}

\usepackage{tikz}
\usepackage{pgf}
\usetikzlibrary{shapes.geometric, arrows, mindmap,shapes,positioning,backgrounds,decorations, matrix}
\usepackage{pgfplots}
\pgfplotsset{width=10cm,compat=1.9}
\usepgfplotslibrary{colormaps,fillbetween}

\theoremstyle{definition}
\newtheorem{Definition}{Definition}[section]

\theoremstyle{plain}
\newtheorem{Theorem}[Definition]{Theorem}
\newtheorem{Corollary}[Definition]{Corollary}
\newtheorem{Lemma}[Definition]{Lemma}

\theoremstyle{remark}
\newtheorem{Remark}[Definition]{Remark}

\theoremstyle{remark}
\newtheorem{Observation}[Definition]{Observation}

\numberwithin{equation}{section}

\usepackage[autostyle]{csquotes}
\usepackage[style=ieee,citestyle=numeric]{biblatex}
\usepackage{authblk}

\title{Mathematical Algorithm Design for Deep Learning under Societal and Judicial Constraints: The Algorithmic Transparency Requirement}

\author[1,2,3]{Holger Boche}
\author[4]{Adalbert Fono}
\author[4,5]{Gitta Kutyniok}
\affil[1]{Institute of Theoretical Information Technology, TUM School of Computation, Information and Technology, Technical University of Munich, Germany}
\affil[2]{Munich Center for Quantum Science and Technology (MCQST), Munich, Germany}
\affil[3]{CASA – Cyber Security in the Age of Large-Scale Adversaries– Exzellenzcluster, Ruhr-Universität Bochum, Germany}
\affil[4]{Department of Mathematics, Ludwig-Maximilians-Universität München (LMU Munich), Germany}
\affil[5]{Munich Center for Machine Learning (MCML), Munich, Germany}
\date{}

\newcommand{\R}{\mathbb{R}}
\newcommand{\Q}{\mathbb{Q}}

\newcommand{\N}{\mathbb{N}}
\newcommand{\Z}{\mathbb{Z}}

\newcommand{\C}{\mathbb{C}}

\newcommand{\abs}[1]{\ensuremath{\left\lvert#1\right\rvert}}
\newcommand{\norm}[2][]{\ensuremath{\left\lVert#2\right\rVert_{#1}}}

\DeclareMathOperator*{\argmin}{arg\,min}

\begin{document}

\maketitle

\begin{abstract}
    Deep learning still has drawbacks in terms of trustworthiness, which describes a comprehensible, fair, safe, and reliable method. To mitigate the potential risk of AI, clear obligations associated to trustworthiness have been proposed via regulatory guidelines, e.g., in the European AI Act. Therefore, a central question is to what extent trustworthy deep learning can be realized. Establishing the described properties constituting trustworthiness requires that the factors influencing an algorithmic computation can be retraced, i.e., the algorithmic implementation is transparent.
    Motivated by the observation that the current evolution of deep learning models necessitates a change in computing technology, we derive a mathematical framework which enables us to analyze whether a transparent implementation  in a computing model is feasible.
    
    We exemplarily apply our trustworthiness framework to analyze deep learning approaches for inverse problems in digital and analog computing models represented by Turing and Blum-Shub-Smale Machines, respectively. Based on previous results, we find that Blum-Shub-Smale Machines have the potential to establish trustworthy solvers for inverse problems under fairly general conditions, whereas Turing machines cannot guarantee trustworthiness to the same degree. 
\end{abstract}


\section{Introduction}
The core idea of machine learning, that is, enable an algorithm to extract relevant information from an available data set to solve a given problem, coupled with an evolution of digital computing technology and power, led to a revolution in a wide range of applications \cite{Senior20DeepFold,He2015DelvingDI,Silver16Go}. Even more, by further augmenting the machine learning models and the data sets, great advances have been made via the deep learning framework in fields such as natural language processing, which were expected to be amenable to this approach to a lesser degree due to their inherent complexity \cite{Brown20GPT3,Lam2023GraphCast}. Deep learning \cite{LeCun15DL, Goodfellow16DL, Berner2021modernMathDL} refers to a specific class of machine learning models, so-called (deep) artificial neural networks \cite{McCulloch43NNs} that are adjusted and optimized on given training data via fairly simple loss functions and basic iterative methods such as stochastic gradient descent along with backpropagation \cite{Rumelhart86BP}.

Therefore, it is widely acknowledged that the success of deep learning can be attributed mainly to three pillars. First, the availability of vast amounts of data enabled the breakthrough of the deep learning approach by outperforming previous methods by a large margin \cite{Krizhevsky2012ImageNetCW}. Second, the advancements in digital computing hardware and the accompanied increase of computational power allowed for the effective processing of large date sets in the training phase resulting in larger and deeper networks that tend to be more capable. Hereby, the initial breakthrough relied on incorporating GPUs in the training of neural networks. By construction, GPUs are better equipped than purely CPU-powered computers to carry out the applied algorithms heavily depending on matrix multiplication operations \cite{Krizhevsky2012ImageNetCW, Pandey22GPU}. The process of increasing training data sets and computing power is still ongoing and cumulated at this stage in digital high performance computing approaches optimized for the implementation of deep learning \cite{silvano2023survey, Jouppi23TPU, Elster22TPU}. Third, the progress in neural network architecture from fully-connected feedforward over convolutional \cite{Krizhevsky2012ImageNetCW} and residual networks \cite{He2016ResidualDNNs} to transformers \cite{vaswani2017selfattention} as well as in training methodology, e.g., incorporating techniques such as self-supervision \cite{balestriero2023cookbook} and reinforcement learning (with human feedback) \cite{christiano2017deepHF}, transferred the potential benefits of larger training sets and more computing power into practical improvements. 

\subsection{Energy and Scaling Limitations of Deep Learning}
However, it is not clear how far the current approaches can be further scaled and improved under the deep learning framework. There are indications suggesting that this development may slow down or even halt \cite{Thompson21DLreturn, thompson2020computational}. The data sets employed to train state-of-the-art large language models are already including a noteworthy fraction of (English) text on the internet \cite{Brown20GPT3}.  
Hence, even larger and suitable data bases need to be generated to train future models by combining different data types such as text, audio and video. Besides, the ongoing digitization of the physical world via sensors, which observe and perceive aspects of their environment -- think of autonomous vehicles for instance --, leads to an accumulation of additional data but also greater demands in storing and throughput capacity \cite{decadal_plan}. 

Therefore, to process the collected data and to apply algorithmic methods such as deep learning more computing capacity is required, i.e., the necessary number of computational steps increases. Since at present there exists a direct connection between the number of computational operations and the total energy consumption of a (digital) computing device \cite{Landauer61Limit,berut2012Limit}, it seems unlikely that the already immense energy consumption of the current deep learning models does not substantially increase in the future unless the applied techniques are structurally adjusted. Hence, dramatically more data and energy efficient methods have to be incorporated 
or the underlying computing and processing paradigm needs to change so that more efficient but equally powerful computations can be carried out. One promising alternative to the present, purely digital computing approach is to incorporate analog devices in the computing pipeline since analog computing offers potential benefits if its fault tolerance can be increased \cite{decadal_plan, computing_plan}.

\subsection{The Need for Trustworthy Deep Learning}
Due to its ongoing advancements, the scope of tasks that can be successfully tackled by deep learning models is ever-increasing. On the one hand, existing models are tweaked so that they are able to adapt to similar tasks of the same complexity. On the other hand, new, more capable classes of models are introduced -- the last one being foundational models \cite{bommasani2022opportunities} such as large language models -- that are able to solve previously unattainable problems. Mainly, current models still impact their environment indirectly by influencing human decisions but not by direct control of the physical environment. The aforementioned large language models are a prime example of these interactions \cite{chang2023survey}. However, the increasing capabilities of deep learning models and the actual goal of artificial general intelligence \cite{Goertzel2014AGI,Roser2023AGItimeline} indicate that the type of interactions may change in the near future, e.g., autonomous vehicles with their sensing and decision-making powered by deep learning will act as physical agents and thereby cross the barrier from indirect to direct interactions with the physical environment. This is reflected by discussions on machine ethics in the context of autonomous driving, i.e., the questions of how algorithms should decide in certain situations \cite{Thornton17Ethics, Karnouskos20Ethics, Geisslinger21Ethics}. While formulating and agreeing on a decision-making framework for (autonomous) physical agents is undeniably important, the question of how to implement and guarantee abidance by the chosen framework is equally relevant.

The latter goal is complicated by the black box, unreliable, non-accountable, and non-robust behaviour of current deep learning models \cite{ras2023explainable, Szegedy14AdvEx, zhang2023sirens,  bastounis2021mathematics, Adcock20gap}. A well-known failure in this regard is the non-robustness of artificial neural networks towards minimal input perturbations, which entails non-reliability on the network's output \cite{Szegedy14AdvEx, Ilyas19AdvExNotBugs, Carlini18AudioAdvEx, Tsipras18RobustnessOdds, Antun2020InstabilitiesDL, Moosavi16DeepFool}. These drawbacks can be summarized by a lack of trustworthiness \cite{BOULEMTAFES20Privacy, He2022Security, Liu21Privacy, Mireshghallah20PrivacySurvey, Willers20Safety} -- an umbrella term for privacy, security, resilience, reliability, and accountability \cite{Fettweis2022Trustworthiness}. Thereby, the notion of trustworthiness includes amongst others aspect such as 
\begin{itemize}
    \item robustness, i.e., resilience against a variety of challenges: changing environment or situations, noisy or incomplete data, and adversarial attacks;
    \item transparency and interpretability, i.e., clear justification and explanation of the decision-making process; 
    \item fairness, ethical compliance, and privacy, i.e., avoidance of biases, equitable treatment of diverse user groups, and secure management of sensitive information;
    \item safety and security, i.e., protection against potential threats and preventing unintended harmful consequences. 
\end{itemize}
Failing to establish trustworthiness in deep learning systems entails that no performance guarantees can be provided or, at least, circumstances may arise in which deep learning systems exhibit unexpected and potentially harmful behaviour. This fact is well-acknowledged, even beyond the science community, since the present and future deep learning applications are expected to impact all of society. Therefore, policy makers already proposed guidelines and regulations that deep learning models need to satisfy. Among the most influential ones are certainly the European AI Act \cite{WebsiteEuCom} and the G7 Hiroshima Leaders Communiqué \cite{CommuniqueG7}, which describe various degrees of requirements and demands with respect to trustworthiness of deep learning systems. In particular, the European AI Act formulates a clear legal framework that might act as a blueprint for further regulation proposals. This raises the question to which extent trustworthiness can be achieved with the deep learning approach. 
The lack of trustworthiness in deep learning is a persistent issue throughout its evolution since the fundamental approach remained unchanged. This is in contrast to other AI methods such as expert systems \cite{Tan2017Expert}, which by design offer trustworthiness benefits, but fail to reach the capabilities offered by deep learning in other areas. Thus, an open problem is whether an adaptation of key components of deep learning may change the trajectory of trustworthiness. 

\subsection{Our Contributions}
Since the remedy of the energy concerns may require the introduction and integration of novel computing technologies, we assess theoretical computational requirements to establish trustworthy deep learning models. For this purpose, different angles need to be considered. 

First, trustworthiness lacks a universally acknowledged formal definition. Therefore, in \Cref{SubSec:TrAI} abstract principles and potential legal structures based on the introduced aspects of trustworthiness are discussed. Thereby, clear requirements on trustworthiness with focus on the transparency condition are posed.  

Second, the capabilities and limits of algorithms with deep learning models being a specific type can only be evaluated with respect to the hardware/computation paradigm. Consequently, we need to take into account the underlying computing model. To that end, in \Cref{SubSec:AlgSol} we present two different computing models -- digital and analog -- and their respective promise for energy efficient performance. By considering idealized mathematical abstractions illustrating the core idea of the computing approaches, we describe conditions that guarantee trustworthy deep learning implementations, i.e., we convert the introduced non-formal trustworthiness conditions into a mathematical framework. 

Third, trustworthiness in deep learning is a very broad topic. Hence, we restrict to a particular use case -- inverse problems -- which is specific enough to allow for a formal treatment but also enables us to draw more general conclusions. We define the considered inverse problem setting and the associated deep learning solvers in \Cref{SubSec:IPandDL}.   

Subsequently, \Cref{Sec:AS_IP} applies the derived framework of trustworthy computations in the inverse problem use case. Hereby, we rely on the findings in \cite{Boche2022LimitsDL} and \cite{Boche2022InvProb}. The results imply that digital and analog computing have diverging capabilities to enable trustworthy algorithms (and thereby deep learning systems): Digital computing modeled by Turing machines \cite{Turing36Entscheidung} has certain limitations that potentially can be avoided by analog computing modeled by Blum-Shub-Smale (BSS) machines \cite{Blum89BSSmachines}. 

\subsection{Limitations}\label{SubSec:Limit}
Can we transfer the observations with regard to trustworthy algorithms from the field of inverse problems to a broader class of deep learning applications? Although each application necessitates an in-depth consideration of its own, we can formulate some general conclusions. The existence of a trustworthy algorithm solving a task may depend on the underlying computing model and different outcomes may indeed arise. In particular, digital computing is not always the optimal choice to achieve trustworthiness, whereas analog computing may enhance the capacity in theory. The decisive question is, whether analog computing that translates theoretical into practical benefits can be realized. We analyzed analog computing under the BSS model, but there is no universally accepted mathematical model precisely describing analog computations (as with the Turing model for digital computations). Hence, further research is required to establish appropriate theoretical models that include, for instance, error correction and approximate analog computing, i.e., computing models that can trade energy and computing time with accuracy (presumably how biological brains operate) \cite{Ulmann2022Analog, Haensch19Analog, Hamerly22Analog, miscuglio2021approximate}. 

The potential of a trustworthy algorithmic computation is evaluated via the notion of algorithmic solvability, which describes a correct, reliable and accountable method to solve a given problem; see \Cref{Sec:AS_IP}. 
Thus, algorithmic solvability may also provide a basis for verifying the abidance of legal requirements, in particular in the field of deep learning. We demonstrate that algorithmic solvability strongly depends on the specific problem formulation so that it may guide us towards descriptions that in principle allow for trustworthy solutions. At the same time, intricate and diverse real world tasks that require advanced and rather general solution techniques may not be amenable to the notion of algorithmic solvability. Either the generality of the tasks prevents algorithmic solvability directly or situations may arise in which decisions have to be made under incomplete or uncertain information so that there is no clear assignment of right and wrong behavior. The latter cases would be difficult to transfer to the introduced formalism and consequently different tools may be necessary to derive statements about trustworthy algorithmic solutions. 

Similarly, algorithmic solvability is certainly not the only approach to determine the trustworthiness of algorithmic methods. The utilized notion of algorithmic solvability and the underlying computability concepts foremost guarantee the (technological) integrity of a system \cite{Fettweis2022Trustworthiness}, which refers to a state in which the system in question resides within its specified margin of operation. Thus, the utilized computing machine does not (inadvertently) interfere with and influence the expected outcome of the performed computation. Therefore, adherence to a framework provided by legal regulations may be ensured via algorithmic solvability. Yet, in principle, the use of AI systems is feasible without adhering to algorithmic solvability if the arising drawbacks are acknowledged. In addition, if human intelligence is understood as an algorithm, it presumably may not fully guarantee algorithmic solvability and associated trustworthiness properties since it is subject to instabilities such as cognitive biases. As a result, a scenario is conceivable in which AI systems prevail without ensuring the aforementioned principles, for example by leading to better results on average than purely human intelligence. A concrete example is autonomous driving, which could prevail as soon as the expectation is reached that autonomous vehicles could improve accident statistics. Nevertheless, the outlined scenario exhibits a severe lack of trustworthiness, which may cause liability issues in case of accidents and ultimately may prevent its occurrence.

\subsection{Potential Impact and Extensions}
In this paper, we propose a mathematical framework for algorithmic trustworthiness considerations based on computability theory, i.e., an approach to assess the possibility of a transparent algorithmic implementation of a problem given a computing model. Although, it is not a universally applicable framework, it provides a step forward in the trustworthiness analysis of deep learning in the following sense: On the one hand, it allows to assess whether transparency may in principle be attainable in a given scenario. Hence, further analysis can evaluate a potential trustworthy implementation or the seriousness of the lack of trustworthiness in a specific scenario needs to be considered. On the other hand, a trustworthiness analysis may steer the implementation of a problem by adapting associated parameters such as computing platform, expected accuracy, and generality so that a potentially trustworthy implementation becomes feasible. This approach can be used in a detail-oriented setting for a specific task as well as on a larger scale that asks for promising directions for establishing trustworthiness in a broader scope.

The analysis in the inverse problem use case indicates that a trustworthy solver either presupposes a sufficiently narrow problem description or if a more general solver is envisioned, then computing capacities beyond digital computing may be necessary or certain limits need to be accepted, i.e., a provable trustworthiness certificate in the introduced framework is not feasible. Thus, our analysis enforces the observation that a shift from purely digital information processing (and especially computing) to novel approaches also comprising analog techniques seems inevitable -- not only due to demands on energy efficiency and data throughput \cite{decadal_plan} but also from the trustworthiness perspective. An emerging example for the changing paradigm is given by neuromorphic computing \cite{Christensen2022NCSurvey} with promising energy and processing gains, whereas a trustworthiness analysis based on adequate computing models is still pending.

\section{Trustworthiness Framework}
\subsection{Societal and Judicial Requirements}\label{SubSec:TrAI}
Trustworthiness is a multifaceted property, where the individual features are partially overlapping and entail one another. Yet, it provides a comprehensive description of qualities potential guidelines may require. Hence, trustworthiness (and the implied abidance by an approved decision-making framework) can be seen as a prerequisite to implement and operate deep leaning systems in certain scenarios, including safety-critical and high-leverage settings with direct influence on the physical environment such as autonomous driving. Due to a lack of technical assurances, abstract principles, codes of conduct, and legal regulations have been proposed such as Algorithmic Transparency, Algorithmic Accountability, and Right to Explanation for technology assessment in the context of trustworthiness \cite{RF1_2023, RF2_2023}: 
\begin{itemize}
    \item Algorithmic Transparency (AgT) refers to the requirement of the factors determining the result of an algorithm-based decision being visible to legislator, operator, user, and other affected individuals.
    \item Algorithmic Accountability (AgA) refers to the question of which party, individual, or possibly system is to be held accountable for harm or losses resulting from algorithm-based decisions, particularly those that are deemed to be faulty.
    \item Right to Explanation (RtE) refers to the right of an individual that is affected by an algorithm-based decision to know the entirety of relevant factors and their specific expression that lead to the decision. 
\end{itemize}
These notions may not capture all relevant nuances discussed in social, judicial, and political science, but they do guarantee or at least approach aspects of trustworthiness covered by robustness, interpretability, fairness, safety, etc. However, at present, there exists no widely accepted technological characterization in form of standards and specifications. The need is accentuated by the described fact that existing deep learning techniques do not result in models which satisfy AgT, AgA, RtE and perform sufficiently well simultaneously. Therefore, the question is whether future methods can abide by these (potential) regulations, and if a positive answer is found to what degree. That is, can we expect future methods to solve trustworthiness issues in a broader context? In particular, are certain aspects that result in non-trustworthy behaviour of deep learning models structurally inherent to the approach or can they be avoided by suitable adaptations? By introducing formal (technical) requirements describing AgT, AgA, and RtE, we can make first strides in answering these questions. Thereby, we focus on AgT since RtE can be considered as a direct application of AgT. Moreover, AgA is not possible without a clear understanding of the algorithm-based decision making provided by AgT. Therefore, AgT is the backbone of the introduced trustworthiness notions. 

An often neglected fact when discussing technological standards is the interplay between hardware platforms for computing, such as digital, neuromorphic, and quantum hardware, and the implemented algorithms representing the software side. Deep learning -- in essence just a specific type of algorithm -- needs to be expressed by a set of instructions in a (programming) language associated to the utilized hardware. Consequently, the capabilities of a given implementation also hinge on the power of the programming language, respectively the employed hardware. Thus, we analyze the potential impact of the hardware platform on trustworthy outcomes characterized by AgT based on two abstract computing models, namely the Turing model \cite{Turing36Entscheidung} for digital computing and the Blum-Shub-Smale (BSS) model \cite{Blum98ComplRealComp} for (idealized) analog computing.

\subsection{Transparency Condition}\label{SubSec:AlgSol}
An algorithmic computation of a problem provides an explicit and reliable approach that is guaranteed to (or clearly describes the degree and the circumstances under which it) succeed(s). In mathematical terms, an algorithm is a set of instructions that operate under the premises of some formal language characterizing the tackled problem. Thus, the specific definition of an algorithm depends on the considered formal language and the underlying computing model, e.g., digital and analog computations, which will be formally introduced in \Cref{subsubsec:dig} and \Cref{subsubsec:ana}, respectively. A real-world physical problem can be translated into a mathematical model that describes its domain with the feasible inputs, outputs and operations of a potential algorithm. Hence, independent of the individual algorithmic steps, we can describe the abstract input-output relation characterized by a function on the identified domain that the algorithm needs to realize. Note that, in general, the domain of the algorithm may differ from the domain of the mathematical model of the physical problem; this behavior occurs, for example, in digital computing. By interpreting the mathematical model as a function describing the input-output relation of the problem, we can rely on the following notion of the realization of an algorithm tackling problems on continuous quantities described by real numbers.  
\begin{Definition}\label{def:realAlg}
    Given a problem described by the input-output relation of a function $f: \R^m \to \R^n$, an \textit{algorithm} $\mathcal{A}$ computing $f$ \textit{realizes a mapping} $\mathcal{A}_f: \R^m \to \R^n$ with $\mathcal{A}_f=f$.  
\end{Definition}
\begin{Remark}
    The realization of an algorithm is particularly important for digital computing, where the computations are performed on representations of real numbers that constitute a specific subset of real numbers; see \Cref{subsubsec:dig} for more details.
\end{Remark}
Subsequently, we will discuss the relevance of the realization of algorithms with respect to AgT. The mathematical model (respectively the derived function) may act as the ground truth and serve as the foundation for the trustworthy requirements introduced in \Cref{SubSec:TrAI}. In particular, the mathematical model provides a basis to assess AgT by identifying the factors influencing the underlying problem. 
Factors outside the mathematical model may not impact the algorithmic computation and outcome, since the transparency of the algorithmic computation can then no longer be guaranteed. Although these considerations may seem trivial, the implementation of the algorithm on a given computing platform, which is subject to mathematical modeling itself, adds another layer of complexity. To execute the algorithm on a suitable hardware platform, the abstract problem, respectively the corresponding mathematical model, needs to be translated into a machine-readable language. 
However, the mathematical model of the problem and the mathematical model of the computing platform do not necessarily agree. 
Thus, the input and output expressions must be unambiguously translatable between these two systems to guarantee a proper implementation of the algorithm abiding AgT. 

To illustrate this issue consider the machine-readable description of a real number such as $\pi$ in the digital computing model.
Due to its infinite binary expansion, $\pi$ can only be represented by finite algorithms that approximate it to any desired precision on digital computers. Despite the difference in description as a mathematical entity and a machine-readable object, we can identify both representations as equally valid and translatable, yet not unambiguous, since a real number may have multiple -- equally valid -- machine-readable descriptions. Besides, it is well-known that not all irrational numbers possess an unequivocal digital representation \cite{Turing36Entscheidung}. Consequently, the properties of the applied computing device need to be taken into consideration when the feasibility of an algorithmic computation as well as its trustworthiness is evaluated.


 
\begin{Definition}\label{def:transAlg}
    An \textit{algorithmic implementation} is \textit{transparent} in a given computing model if the realization $\mathcal{A}_f$ of some function $f: \R^m \to \R^n$ by an algorithm $\mathcal{A}$ is not altered by its implementation in the computing model. We then say that $f$ allows for a \textit{transparent algorithmic implementation} in the given computing model.
\end{Definition}
\begin{Remark}
    In general, if a real-world problem $P$ can be expressed as a function $f$, then the input domain of $f$ necessarily represents every factor determining the outcome of $P$. Hence, a transparent algorithmic implementation of a closed-form expression of $f$ guarantees that $P$ can be solved by an algorithm abiding AgT in the considered computing model since the realization of the algorithm solely relies on the factors constituting the problem. Consequently, the algorithmic decision making can in principle be retraced so that providing AgT is feasible. However, in deep learning a potential solver for $P$ is sought based on a (imperfect) data set describing $P$ so that $f$ may not posses a closed-form expression or may be unknown, e.g., not all relevant factors determining $P$ are identified. Thus, the contribution of the individual factors affecting the solver may not be apparent due to the black box behaviour of deep learning so that assessing AgT requires a successive analysis, which presupposes a transparent algorithmic implementation. Therefore, a transparent algorithmic implementation can be seen as a minimum requirement to obtain AgT; the actual conditions depend on the chosen computing model and its formalization of an algorithm. For instance, in the digital case transparency requires the algorithm to be independent of the specific representation of a real number. 
\end{Remark}
Just constructing a transparent algorithm is not sufficient, since besides being comprehensible we also expect the algorithm to deliver a solution of the considered problem. Therefore, it is immanent to maintain the integrity between the mathematical model of the problem and the applied algorithm, i.e., the algorithmic solution needs to unequivocally reflect the `true' solution of the mathematical model.
\begin{Observation}\label{obs:AgT}
    To attain AgT and provide the correct output, an algorithm must reside within its specified margin of operation, i.e., the algorithm preserves the input-output relation of the underlying problem specified by a mathematical model (respectively the derived function).  
\end{Observation}
In the subsequent analysis, our focus does not reside on practical limitations of real-world hardware related to computation time, memory and energy consumption. Via mathematical models of the considered computing paradigm -- the most prominent one certainly being digital computations --, we study the possibility of theoretical guarantees with respect to trustworthiness.  


\subsubsection{Digital Computations}\label{subsubsec:dig}

The concept of digital machines is encapsulated by the mathematical model of Turing machines \cite{Turing36Entscheidung}. In fact, the widely accepted Church-Turing Thesis \cite{Copeland2020CTT} implies that Turing machines are a definitive model of digital computers, describing their (theoretical) capabilities perfectly. Thus, Turing machines provide a framework for analyzing digital computations by taking into account its inevitable approximate behaviour with respect to irrational numbers. More exactly, Turing machines introduce a notion of effective computations in a finite number of steps on real numbers. Thereby effectiveness refers to the condition that a Turing machines not only computes an approximate solution but also guarantees that the solution is within some previously prescribed error bound, which can be chosen arbitrarily small. Hence, the reliability and correctness of an obtained algorithmic solution is guaranteed by design.

Next, we shortly formalize effective computations via \textit{recursive functions} \cite{Kleene36Recursive}, which constitute a special subset of the set $\bigcup_{n=0}^\infty \{f : \N^n \hookrightarrow \N \}$, where '$\hookrightarrow$' denotes a partial mapping. Recursive functions coincide with the functions $f : \N^n \hookrightarrow \N$ that are computable by Turing machines, i.e., there exists a Turing machine that accepts input $x \in \N^n$ only if $f(x)$ is defined, and, upon acceptance, computes $f(x)$ \cite{Turing37Equivalence}. 
\begin{Lemma}
    A function $f : \N^n \hookrightarrow \N$ is a recursive function if and only if it is computable by a Turing machine. 
\end{Lemma}
Using recursive functions, we can identify the set of effectively computable real numbers. By introducing machine-readable description of real numbers, which was exemplarily demonstrated for $\pi$ in \Cref{SubSec:AlgSol} and can be formally expressed via recursive functions \cite{Kleene36Recursive}, one can construct sequences of rational numbers converging (with error control) to the actually considered real numbers. In this way, the computations performed by Turing machines are reduced to the rational domain, where exact computations are feasible, but at the same time the accumulated error due to the non-exact representation of real numbers is controlled. 
\begin{Definition}\label{def:CompSeqRat}
    A sequence $(r_k)_{k \in \N}\subset \Q$ of \textit{rational numbers} is \textit{computable}, if there exist three recursive functions $a, b, s : \N \to \N$ such that $b(k) \neq 0$ and
    \begin{equation*} 
        r_k = (-1)^{s(k)} \frac{a(k)}{b(k)} \qquad \text{ for all } k \in \N. 
    \end{equation*}
     A \textit{real number} $x\in\R$ is \textit{computable}, if there exists a computable sequence $(r_k)_{k \in \N}$ of rationals such that
    \begin{equation*} 
       \abs{r_k - x} \leq 2^{-k} \qquad \text{ for all } k \in \N.
    \end{equation*}
    We refer to the sequence $(r_k)_{k \in \N}$ as a \textit{representation} for $x$.
\end{Definition}
\begin{Remark}
    The definition can be straightforwardly extended to vectors and complex numbers by considering each component and part individually, respectively.
\end{Remark}
Having established a formal notion of an algorithm via Turing machines, we can specify the characterization of a transparent algorithm from \Cref{def:transAlg}. The key observation is that any Turing machine strictly operates on rational representations although the tackled problem may reside in the real domain, i.e., Turing machines map input representations to (computable) sequences. Hence, we can identify a Turing machine TM with an associated mapping $\Psi_{\text{TM}}: R \to S$, where $R$ and $S$ denote the representation space and the space of (computable) sequences, respectively, defined as
\begin{equation*}
    R \coloneqq\{ (r_k)_{k \in \N} : (r_k)_{k \in \N} \text{ is a representation of some } x \in \R\} 
\end{equation*}
and
\begin{align*}
    S &\coloneqq \{ (s_k)_{k \in \N} : (s_k)_{k \in \N}\subset \Q \text{ is sequence computed by a Turing machine on input} \\
    &\qquad\qquad\qquad\text{of a representation of some } x \in \R\}.
\end{align*}
However, the transfer between the different domains -- the representation space and the real numbers -- shall not impact the algorithmic computation in order to preserve transparency.
\begin{Definition}
     An \textit{algorithmic implementation} of a problem in the real domain is \textit{transparent in the Turing model} if the associated Turing machine TM operates consistently on every input in the following sense: For any two representations $(r^1_k)_{k \in \N}, (r^2_k)_{k \in \N}$ of an input instance $x \in \R^m$, the output sequences $\Psi_{\text{TM}}((r_k^1)_{k \in \N})$ and $\Psi_{\text{TM}}((r_k^2)_{k \in \N})$ encode the same outcome in the underlying (real) problem domain, i.e., the outcome exclusively depends on $x$ but not on its specific representation and other factors.
 \end{Definition}
\begin{Remark}
    Consider a toy example in which an algorithm $\mathcal{A}$ realizes the real-valued function $f(x) = ax +b$ for some constants $a,b \in \R$. Then transparency requires that the implementation of $\mathcal{A}$ in the Turing model is not only independent of the representation of inputs $x\in \R$, but additionally independent of the representation of the constants $a,b$. 
\end{Remark}
Finally, observe that a Turing machine TM with an associated mapping $\Psi_{\text{TM}}$ does not necessarily constitute a well-defined algorithm $\mathcal{A}$ with realization $\mathcal{A}_f$ on the real domain, i.e., $\Psi_{\text{TM}}((r_k)_{k \in \N})$ may not represent real number for an admissible input $(r_k)_{k \in \N}$. The notion of computable functions, which essentially describes the effective computation of functions by Turing machines on real numbers, circumvents this issue.
\begin{Definition}
    A function $f: \R^m \to \R^n$ is \textit{Borel-Turing computable}, if there exists a Turing machine that transforms each representation of a vector $x \in \R^m$ into a representation for $f(x)$.
\end{Definition}
\begin{Remark}
   There exists different notions of computable functions on real numbers. We refer to \cite{AvigadBrattka14CompAnal, Pour-El17Computability, Weihrauch00CompAnal} for an in-depth treatment of the topic and we highlight only the key properties of computable functions that help in providing an intuitive understanding. The need for approximate computations is closely related to the real-valued domain of $f$. For functions operating on the natural or rational numbers, exact computations can in principle be expected. Moreover, Borel-Turing computability can also be applied to complex-valued functions by identifying real and imaginary parts as real numbers. 
\end{Remark}
\begin{Remark}
    Borel-Turing computability implies the existence of an algorithm $\mathcal{A}$ realizing a mapping $\mathcal{A}_f: \R^m \to \R^n$, i.e.,  $\mathcal{A}_f = f$, via the associated Turing machine. From a practical point of view, Borel-Turing computability can be seen as a requirement for an algorithmic computation of the input-output relation of a problem (described by a function) on perfect digital hardware in the following sense. In particular, the associated Turing machine, i.e., the algorithm, takes input representations and determines a sufficient input precision, i.e., suitable elements of the representations, so that the performed computation will terminate once an output within a prescribed worst-case error bound $\varepsilon>0$ is obtained. Here, the input representation is itself a Turing machine, which can be queried with a precision parameter and provides an approximation of the ``exact'' input. In theory, by iteratively calling the algorithm with a declining sequence of error bounds $\varepsilon_k = 2^{-k}$ for a fixed representation of an input $x\in \R^m$, one would obtain a (computable) representation encoding $f(x)$. 
\end{Remark}
\begin{Remark}
    In general, the representation of a computable vector $x$ is not unique. Hence, the representation of a Borel-Turing computable function $f$ at $f(x)$ may depend on the representation of $x$ given as input to the Turing machine. A small wrinkle is added by the fact that not every real number has a description based on recursive functions, i.e., only the computable real numbers (indeed a proper subset of the real numbers) are considered as admissible inputs in the Borel-Turing setting. In contrast, any real number can be approximated by a convergent sequence of rational numbers so that the concept of Borel-Turing computability can be extended to the whole real number domain under certain conditions. For our needs, these subtle differences and their implications can be neglected and we apply the notion of Borel-Turing computability introduced in the definition, but we point to \cite{colbrook21stable} for the treatment of the inverse problem use case under the adapted notion. 
\end{Remark}
Aside from Turing machines further abstractions of digital computations have been established, for instance, the Blum-Shub-Smale (BSS) machines represents a common heuristic formalization (in contrast to the precise model of Turing machines according to the Church-Turing thesis) \cite{Blum98ComplRealComp}. Therefore, BSS machines do not provide a suitable starting point for our intended trustworthiness considerations on digital hardware by design. At the same time, it turns out the BSS framework may be suitable in a different context via their relation to analog hardware. Although there does not exist a widely accepted formalization equivalent to the Turing model for other types than digital hardware, the BSS framework is a candidate to (abstractly) model several forms of analog computing \cite{Grozinger2019Biocomp}. 

\subsubsection{Analog Computations}\label{subsubsec:ana}

The study of analog hardware gained considerable traction in the last years due to the rapidly increasing demand for energy and storage of digital information processing and computing \cite{thompson2020computational, Thompson21DLreturn}.
Even more, convincing arguments indicate that the scaling of computation and information processing at the current level is not sustainable in the near future if the same (digital) technologies are applied, i.e., a technological disruption based on the introduction of new (analog) approaches is necessary \cite{decadal_plan, computing_plan}. For instance, innovative memory and storage technology as well as novel approaches in world-machine interfaces that can sense, perceive, and reason based on low operational power and latency are required. This may only be realizable by incorporating analog (electronic) components in the (currently mostly digital) computing and information processing pipeline. A prime example is provided by neuromorphic computing and signal processing systems \cite{Esser15Neuromorphic, Smith22Neuromorphic, Maas22Neuromorphic, Markovic20Neuromorphic, Blouw20Neuromorphic, Schuman22Neuro}. 

Neuromorphic systems are inspired by the structure and information processing of biological neural networks and can be realized in analog, digital, and mixed analog-digital fashion \cite{Christensen2022NCSurvey}. 
Digital computers typically follow the von Neumann architecture \cite{Aspray90JvN} that leads to inherently large time and energy overhead in data transport \cite{Backus78Bottleneck, efnusheva2017survey}. In contrast, neuromorphic computers incorporate emerging concepts such as `in-memory computing', which avoid these bottlenecks by design \cite{Boybat21InMemory,Karunaratne20InMemory, Sebastian20InMemory,Payvand19InMemory}. At present, the main advantage of neuromorphic systems are the expected savings in energy consumption, in particular, by deploying artificial intelligence applications such as deep learning on neuromorphic hardware \cite{Esser15Neuromorphic, Smith22Neuromorphic, Maas22Neuromorphic, Markovic20Neuromorphic, Blouw20Neuromorphic, Papp2021NanoscaleNN}. Further promising (but not yet realizable) analog computing paradigms comprise biocomputing \cite{Grozinger2019Biocomp, Wagenbauer2017DNAassembly, Poirazi2017Dendritic, Wright2022DeepPhysycalNN} and (analog) quantum simulation \cite{Bloch22Quantum, Bloch22Quantum2}.

Moreover, analog computing may also provide benefits with regard to the computability of certain problems. 
Important tasks in information theory, signal processing, and simulation are not Borel-Turing computable \cite{Elkouss2018MemoryEC, Schaefer2019TuringMS, Boche20SpecFac, Boche20BandlimitedSignals, Boche21PeakValue, PourEl97WaveEq, Boche20LTI}. whereas, computability on BSS machines has been established for certain applications \cite{Boche2021DoSAttacks, Boche2021DetectDoS, Boche2022RemoteState}. 
Computability in BSS sense conveys the same concept as computability on digital devices, but under a more general framework. BSS machines carry over complexity theory in the Turing machine model to a larger variety of structures by operating on arbitrary rings or fields, even infinite fields such as $\R$ are feasible. In addition, BSS machines on $\Z_2 = \{ \{0,1\},+,\cdot \}$ recover the theory of Turing machines. Thus, BSS machines are a generalized abstraction of Turing machines that are structure-wise similar. Just as Turing machines, BSS machines operate on an infinite strip of tape according to a program illustrated by a finite directed graph with different types of nodes associated to operations such as input processing, computing, branching and output processing. 
For a detailed introduction and comparison, we refer to \cite{Blum98ComplRealComp,Blum04CompoverReals} and the references therein. We only wish to highlight that a BSS machine, which operates on $\R$, processes real numbers (as entities) and performs field operations (`$+$',`$\cdot$') via compute nodes and comparisons (`$<$',`$>$',`$=$') via branch nodes exactly. Thus, BSS machines offer a mathematical framework to investigate analog real number processing and computation. Hereby, BSS computable functions are simply input-output maps of BSS machines, i.e., the set of BSS computable functions precisely characterizes functions that can be computed (in finite time) by algorithmic means in the BSS model.
\begin{Definition} 
    A function $f: \R^m \to \R^n$ is \textit{BSS computable} if there exists a BSS machine with an input-output relation described by $f$.
\end{Definition}
\begin{Remark}
    The output $\Psi_{\mathcal{B}}(x)$ of a BSS machine $\mathcal{B}$ is defined if $\mathcal{B}$ according to its program terminates its calculations on input $x \in \R^m$ after a finite number of steps. The hereby introduced map $\Psi_{\mathcal{B}} : \R^m \to \R^n$ is the input-output function of $\mathcal{B}$, which directly relates to the realization of an algorithm (i.e., BSS machine) introduced in \Cref{def:realAlg}.  
\end{Remark}
\begin{Remark}\label{rm:ComplexRepresentation}
    The definition can be straightforwardly extended to complex functions, however, the representation of the complex field impacts the capabilities of corresponding BSS machines and thereby also the set of computable functions. BSS machines that treat complex numbers as entities can not perform comparisons of arbitrary complex numbers but only check the equality to zero at their branch nodes due to the fact that $\C$ is not an ordered field. As a consequence, elementary complex functions such as $z \mapsto \Re(z)$, $z \mapsto \Im(z)$, $z \mapsto \bar{z}$, and $z \mapsto \abs{z}$ are not BSS computable in this setting \cite{Blum98ComplRealComp}. Identifying $\C$ with $\R^2$ instead and employing BSS machines that take complex inputs $z$ in form of $(\Re(z), \Im(z))$ entails that $z \mapsto \Re(z)$ and $z \mapsto \Im(z)$ are computable since the corresponding BSS machine only needs to process the respective part of the representation of $z$.
\end{Remark}



The crucial property of the BSS framework is the handling of the elements of the associated ring or field as entities, which implies the exact storing and processing of real numbers in this computation model. Thus, the real BSS model can not be implemented on digital hardware. It is even unclear if and to what degree a computing device realizing the real BSS model can be constructed by (future) hardware technology due to physical constraints \cite{Bekenstein81PhysBnd}. For instance, random noise in physical processes, which execute the mathematical operations in a hypothetical computing device, complicates or in the worst-case prevents exact processing and computation. Moreover, the BSS model strongly focuses on algebraic properties, which in turn leads to the non-computability of trigonometric, logarithmic and root functions in the BSS model on the real numbers. These functions are typically considered as elementary functions that are expected to be computable in a practical and useful computing model (as indeed is the case in the Turing model). Therefore, real BSS machines are a strongly idealized model and they may not capture the true capabilities of forthcoming analog computing devices so that theoretical benefits may not turn into practical ones. Nonetheless, the study of BSS computable functions provides to a certain degree an outline on the limits of analog computations. 

\subsubsection{Computability and Trustworthiness}
Next, we will analyze under which conditions a trustworthy algorithmic solution of a problem described by an input-output relation, i.e., an associated function, can be expected. Thereby, the crucial step is to establish AgT since it is the basis for various trustworthiness considerations; see \Cref{SubSec:TrAI}. In \Cref{obs:AgT} we derived via \Cref{def:transAlg} a necessary prerequisite to obtain AgT. Applying this framework to the Turing model, we can establish a necessary condition for a transparent algorithm on digital hardware. 
\begin{Lemma}\label{lemma:transparency}
    Given a problem with an input-output relation described by function $f: \R^m \to \R^n$, let $\mathcal{A}$ be an algorithm implemented on a Turing machine. If the algorithmic implementation is not transparent, then $\mathcal{A}$ does not realize $f$. 
\end{Lemma}
\begin{proof}
    Assume the algorithmic implementation of $\mathcal{A}$ on a Turing machine TM is not transparent. Thus, by definition there exist two representations of some input instance $x\in \R^m$ such that the computed output sequences by TM do not encode the same real number. Therefore, at least one of the (computable) output sequences does not represent $f(x)$. Hence, $\mathcal{A}$ does not realize $f$. 
\end{proof}
We can immediately infer from \Cref{lemma:transparency} that Borel-Turing non-computability of a function prevents the existence of an algorithm complying with transparency in the digital computing model so that the following statement holds.
\begin{Theorem}\label{thm:EquivTransCompTuring} 
    There exists an algorithm $\mathcal{A}$ with transparent implementation in the Turing model realizing $\mathcal{A}_f$ if and only if $f: \R^m \to \R^n$ is Borel-Turing computable.
\end{Theorem}
\begin{Remark}
    An algorithm adhering to AgT is necessarily transparent so that Borel-Turing computability of the tackled problem is a prerequisite for AgT. In contrast, Borel-Turing non-computability of a problem implies that any algorithmic approach implemented on digital hardware will have unavoidable flaws or at least certain limits: For any algorithm there exists a representation of some input $x\in \R^m$ such that the computed output sequence does not converge at all or does not converge to $f(x)$. Crucially, the integrity between the mathematical model of the problem and the mathematical model of the computing platform is lost so that AgT via \Cref{obs:AgT} can not be guaranteed.
\end{Remark}

A similar line of reasoning can be adopted to BSS machines and BSS computable problems. In the BSS model algorithms directly operate on real numbers so that each real number is uniquely represented by itself. Thus, any problem that is BSS computable by definition fulfills the transparency condition. In contrast, transparency of an algorithm successfully solving a problem immediately implies BSS computability of the problem.
\begin{Theorem}\label{thm:EquivTransCompBSS}
    There exists an algorithm $\mathcal{A}$ with transparent implementation in the BSS model realizing $\mathcal{A}_f$ if and only if $f: \R^m \to \R^n$ is BSS computable.
\end{Theorem}

Hence, by studying the computability of a problem or, more accurately, the computability of a function describing the problem in mathematical terms, we can assess the existence of trustworthy algorithms (based on AgT) in a formal manner. Thus, we obtained a precise tool to decide whether trustworthiness in an algorithmic computation can be attained.

Besides computability, different techniques to assess trustworthiness in deep learning approaches certainly exist and to decide which one is most suitable for specific circumstances is still an unresolved question \cite{Olah22MechInt, Kaestner23ExplAI,WU22Hitl, Biondi20Certification, Zhang20Certification, Mirman21Certification}; for more details see \Cref{SubSec:Limit}. Yet we will apply the introduced framework to a specific use case of deep learning -- finite dimensional inverse problems --, aiming to derive broadly applicable observations.

\section{Use Case for Trustworthiness Analysis}\label{SubSec:IPandDL}
\subsection{Inverse Problems}
Inverse problems in imaging sciences, i.e., image reconstruction from measurements, is a recurrent task in industrial, scientific, and medical applications such as magnetic resonance imaging (MRI) and X-Ray computed tomography (CT), where the measurements are acquired by the Fourier and Radon transform, respectively.
\begin{Definition}
    An \textit{inverse problem} in the finite-dimensional, underdetermined and linear setting can be formulated as: 
    \begin{equation}\label{eq:problem}
        \text{Given noisy measurements }  y = Ax + e \in \C^m \text{ of } x \in \C^N, \text{ recover } x,
    \end{equation}
    where $A \in \C^{m \times N}, m< N$, is the \textit{sampling operator}, $e \in \C^m$ is a noise vector, $y \in \C^m$ is the \textit{vector of measurements}, and $x \in \C^N$ is the object to recover.   
\end{Definition}
\begin{Remark}
    In the setting of the definition, a typical object to recover is a vectorized discrete image. Furthermore, the underdetermined setting $m<N$ with a limited number of measurements is standard in practice due to time, cost, power, or other constraints.  
\end{Remark}    
Due to the ill-posedness of \eqref{eq:problem} a typical solution strategy is to consider a mathematically more tractable description via an optimization problem. The simplest form is given by the least-squares problem 
\begin{equation*}
    \argmin_{x\in C^N} \norm[\ell_2]{Ax-y}. 
\end{equation*}
Although the solution map is considerably simpler than the one of the original problem \eqref{eq:problem}, the solution is generally not unique. By adding regularization terms to the optimization problem, one tries to steer the optimization process towards favorable solutions. For instance, sparse solutions tend to possess desirable properties but explicitly enforcing them via the $\ell_0$ norm is typically intractable. However, incorporating regularization terms that promote sparsity in the recovery resulted in various solution techniques \cite{Daubechies04SparseReg, Wright2009SparseRec, Cotter2005SparseSol, Selesnick2017SRviaCA, Candes05DecLP, Candes06RobUnc, Candes06UnivEncStrat, Donoho06CompSens, Ji2008BayesianCS, Duarte2011StructuredCS, Elad2007ProjCS} for common approaches such as (quadratically constrained) \textit{basis pursuit} \cite{Candes06Stable, Chen1998AtomicDec} 
\begin{equation}\label{eq:sparseprob}
    \argmin_{x \in \C^N} \norm[\ell_1]{x} \text{ such that } \norm[\ell_2]{Ax -y} \leq \varepsilon \tag{BP}
\end{equation}
and unconstrained square root \textit{lasso} \cite{Tropp06Relax, Belloni11SquareRootLasso}
\begin{equation}\label{eq:lasso}
    \argmin_{x \in \C^N} \lambda \norm[\ell_1]{x} + \norm[\ell_2]{Ax -y}, \tag{Lasso}
\end{equation}
where the magnitude of $\varepsilon>0$ and $\lambda>0$ controls the relaxation, respectively. 
In recent years, deep learning techniques led to a paradigm shift and were established as the predominant method to tackle inverse problems \cite{Zu18AutoMap, Arridge2019SolvingIP, Bubba19Shearlet, Yang16MRIDL, Hammernik18MRIDL2, Chen2018LowLightPhoto, Rivenson17DLmicroscopy, Araya18DLtomography}. The core idea behind the deep learning approach is to learn the underlying relations of the reconstruction process based on data samples. 

\subsection{Deep Learning}

In deep learning a structure called \textit{(artificial) neural network}, which is loosely inspired by biological brains, is employed to approximate an unknown function via a set of given input-output value pairs. In essence, a neural network is a parameterized mapping with properties and capabilities depending on its specific design \cite{Krizhevsky2012ImageNetCW, He2016ResidualDNNs, vaswani2017selfattention}. The simplest form, which we will focus on in the remainder, are feedforward neural networks.
\begin{Definition}
    A \textit{(feedforward) neural network} $\Phi : \R^d \to \R^k$ is given by
    \begin{equation}\label{eq:NNdef}
        \Phi(x) = T_L \rho(T_{L-1}\rho(\dots \rho(T_1 x))), \quad x\in \R^d,
    \end{equation}
    where $T_{\ell} : \R^{n_{\ell- 1}} \to \R^{n_{\ell}}$, $\ell=1,\dots,L$, are affine-linear maps
    \begin{equation*}
        T_{\ell} x = W_{\ell} x + b_{\ell}, \quad W_{\ell} \in \R^{n_{\ell} \times n_{\ell-1}}, b_{\ell} \in \R^{n_{\ell}} \text{ with } n_0 = d, n_L = k, 
    \end{equation*}
    and $\rho: \R \to \R$ is a non-linear function acting component-wise on a vector. The matrices $W_{\ell}$ are called \textit{weights}, the vectors $b_{\ell}$ \textit{biases}, and the function $\rho$ \textit{activation function}, with a common one being the basic \textit{ReLU activation} $\rho(x) = \max\{0,x\}$.
\end{Definition}
\begin{Remark}
    In addition, a neural network can easily be adapted to work with complex-valued inputs by representing them as real vectors consisting of the real and imaginary parts. 
\end{Remark}
By adjusting the network’s parameters, i.e., its weights and biases, according to an optimization process on available data samples, the network ideally learns to approximate the sought function. This process -- the standard technique is to apply stochastic gradient descent coupled with backpropagation \cite{Rumelhart86BP} -- is usually referred to as the training of a neural network. For an in-depth overview of deep learning, we refer to \cite{LeCun15DL, Goodfellow16DL, Berner2021modernMathDL}.

\subsection{Deep Learning for Inverse Problems}

Turning to inverse problems, deep learning techniques can be incorporated in the solution approach in various ways \cite{Ongie20DLInvProb}. The most fundamental and generally applicable approach is to directly learn a mapping from measurements $y$ to reconstructions $x$ without making any problem specific assumptions. Hence, the goal is to obtain a neural network that for some fixed sampling operator $A \in \C^{m \times N}$ and optimization parameter $\mu > 0$ approximates the reconstruction map
\begin{equation}\label{eq:Xi_Aeps}
    \Xi_{P,A,\mu}: \C^{m} \rightrightarrows \C^N,\quad y \mapsto P(A,y,\mu), 
\end{equation}
where $P(A,y,\mu)$ represents the set of minimizers of an optimization problem $P$ given a measurement $y \in \C^m$. For instance,  $P$ is described in \eqref{eq:sparseprob} for basis pursuit with $\mu \coloneqq \varepsilon$. Observe that the reconstruction map is in general set-valued, denoted by `$\rightrightarrows$', since the corresponding optimization problem does not posses a unique solution. Therefore, it is not entirely correct to state that the goal is to compute a neural network that approximates the mapping $\Xi_{P,A,\mu}$. We do not expect a neural network to reproduce all minimizers for a single input but it suffices if the network approximates one specific minimizer. In \Cref{Sec:AS_IP}, we will return to and clarify this issue. 

Ideally, the training process results in a neural network that is capable of solving instances of a particular inverse problem defined by the sampling operator $A$, the optimization problem $P$, and the optimization parameter $\mu$. More powerful would be a neural network that can solve generic inverse problems under specific conditions, e.g., a network that approximates the reconstruction map 
\begin{equation}\label{eq:Xi_mN}
    \Xi_{P,m,N}: \C^{m\times N} \times \C^m \times \R_{>0} \rightrightarrows \C^N,\quad  (A,y,\mu) \mapsto P(A,y,\mu) 
\end{equation}
of any inverse problem of dimension $m\times N$ corresponding to an optimization problem $P$. One can generalize the objective further by allowing the dimension and/or the optimization problem as additional inputs. However, it is a priori not even clear if networks that approximate $\Xi_{P,A,\mu}$ and $\Xi_{P,m,N}$ do exist and can be found via the described deep learning framework. Whereas the former problem can be approached by an expressivity analysis of neural networks, which is already supported by a large body of literature \cite{Cyb89, Hor91, Boelcskei19OptApp, ES16, PMRML17, DDFHP19, Gribonval22AppSpaces, YAROTSKY2017ErrBds}, the latter problem is more intricate. Despite the existence of the sought networks, obtaining them based solely on data samples may not be feasible or the computation process may result in networks with unfavorable properties such as a lack of trustworthiness. Hence, we assess the possibility of computing neural networks that solve inverse problems under the introduced computability framework.

\section{Algorithmic Solvability of Inverse Problems}\label{Sec:AS_IP}
It is intuitively clear that the training and inference of neural networks are distinct problems with varying difficulty. For inverse problems the goal of the training process is to obtain a neural network, which approximates the mapping from measurements to the original data. In other words, one is interested in finding a neural network, i.e., suitable weights and biases, that realizes the mapping in \eqref{eq:Xi_Aeps} or \eqref{eq:Xi_mN}. We focus on the more general case \eqref{eq:Xi_mN}, where the associated optimization problem is given by \eqref{eq:sparseprob} or \eqref{eq:lasso}. Before turning to the question if the desired network can be computed algorithmically, we wish to remark that once obtained the execution of said network on a given input can be performed reliably.
\begin{Theorem}[\cite{Boche2022LimitsDL,Boche2022InvProb}]\label{thm:inference}
    A neural network $\Phi$ as defined in \eqref{eq:NNdef} is a Turing \!\slash BSS computable function given that the activation function $\rho: \R \to \R$ is Turing \!\slash BSS computable.    
\end{Theorem}
\begin{proof}[Proof Sketch]
    This follows from the fact that under the given conditions $\Phi$ is a composition of computable functions in both computing models.
\end{proof}
\begin{Remark}
    Note that we can extend the observation in the theorem to more advanced architectures such as convolutional networks \cite{lecun1995convolutional}. In particular, the statement holds for any network architecture that is composed of basic computable building blocks, which indeed is true for many common variants. Furthermore, standard activations applied in practice such as ReLU are indeed computable.
\end{Remark} 
Does a similar statement hold for the training phase? First, note that the mapping \eqref{eq:Xi_mN} targeted by the training process is in general multi-valued since the solution of the associated optimization problem does not need to be unique. Thus, it neither fits in the introduced deep learning nor computability framework. However, we can circumvent this issue by establishing suitable single-valued functions that fit in both frameworks. The underlying idea is that typically one is not interested in the whole solution set, but one particular element of the solution set or even an element reasonably close to the solution set suffices. Hence, of interest is not to compute the entire set described by the map $\Xi_{P,m,N}$ for given input $(A,y,\mu)$ in \eqref{eq:Xi_mN} but to compute one element of $\Xi_{P,m,N}(A,y,\mu)$, i.e., exactly one minimizer of the optimization problem $P$. In particular, it is not relevant which of the (possibly infinitely many) minimizers is obtained, since any of those is an appropriate solution.

Formally, this concept can be captured by single-valued restrictions of a multi-valued function $f:\mathcal{V} \rightrightarrows \mathcal{Z}$: For each input $v \in \text{dom}(f)$ there exists at least one element $z_v \in f(v) \subset \mathcal{P}(\mathcal{Z})$ so that the map
\begin{align*}
    f^s: \mathcal{V} &\to \mathcal{Z}, \quad v \mapsto z_v    
\end{align*}
is well-defined. We denote by $\mathcal{M}_f$ the set of all the single-valued functions associated with the multi-valued function $f$, i.e., all single-valued functions $f^s$ that are formed by restricting the output of a multi-valued map $f$ to a single value for each input. 
\begin{Definition}\label{def:AlgSolv}
    A problem with an input-output relation described by a multi-valued function $f:\mathcal{X} \rightrightarrows \mathcal{Y}$ is \textit{algorithmically solvable on a BSS or Turing machine} if there exists a function $f^s \in \mathcal{M}_f$ that is computable on a BSS or Turing machine, respectively. 
\end{Definition}
By applying the notion of algorithmic solvability, we reduced our task to evaluate the computability of well-defined single-valued functions. It may be the case that most functions in $\mathcal{M}_f$ are non-computable. However, as long as there exists at least one computable function among them the task is deemed algorithmically solvable.
\begin{Remark}
    The presented approach is not the only viable option to assess computability of a multi-valued mapping $f$. The (non-)existence of algorithms can also be established via the distance to the solution set measured by an appropriate metric. Therefore, a hypothetical algorithm does not approximate a fixed element of the solution as the admissible distance to the solution set varies, which is the case in our approach via single-valued restrictions. Thus, a problem characterized by $f$ may not be algorithmically solvable in the introduced sense, but an algorithm obeying this distance description may still exist. Note that this case only arises if $f$ is indeed a multi-valued and not a single-valued function. Moreover, the notion of algorithmic transparency needs to be adjusted to cover 'the distance to the solution set' approach. We refer to \cite{colbrook21stable} for more details, where algorithmic solvability via this notion is pursued in the inverse problem setting. 
\end{Remark}
Finally, we can apply the framework of algorithmic solvability to inverse problems described via basis pursuit and square root lasso. Indeed, we find differences in algorithmic solvability in the Turing and BSS model as we now detail.

\subsection{Algorithmic Non-Solvability of Inverse Problems in Turing Model}
In the Turing setting, for a certain range of optimization parameters we not only establish algorithmic non-solvability but even non-approximability.
\begin{Theorem}[\cite{Boche2022LimitsDL}]\label{thm:nonApprox}
    Consider the optimization problems \eqref{eq:sparseprob}, \eqref{eq:lasso} and the associated mappings $\Xi_{\text{BP},m,N}(\cdot,\cdot,\varepsilon)$ and $\Xi_{\text{Lasso},m,N}(\cdot,\cdot,\lambda)$, where $N \geq 2$ and $m < N$, for fixed parameters $\varepsilon \in (0,\sfrac{1}{4})$ and $\lambda \in (0,\sfrac{5}{4}) \cap \Q$, respectively. The problems described by $\Xi_{\text{BP},m,N}(\cdot,\cdot,\varepsilon)$ and $\Xi_{\text{Lasso},m,N}(\cdot,\cdot,\lambda)$ are not algorithmically solvable on Turing machines.
\end{Theorem}
\begin{proof}[Proof Sketch]
    The main step is to characterize algorithmic non-solvability conditions for functions $\Xi_{P,m,N}(\cdot,\cdot,\mu)$, introduced in \eqref{eq:Xi_mN} based on the solution set of the optimization problem $P$ with optimization parameter $\mu$. The idea is to `encode' a recursively enumerable but non-recursive set $B \subset \N$ in the domain of $\Xi_{P,m,N}(\cdot,\cdot,\mu)$, i.e., a set such that there exists a Turing machine that takes numbers $n\in\N$ as input and confirms (after a finite amount of time) that $n$ is an element of $B$, if $n\in B$ does indeed hold, but fails to decide whether $n\in B$ or $n\in B^c$ in general. To that end, a computable sequence $(\xi_n)_{n\in\N} \subset \Xi_{P,m,N}(\cdot,\cdot,\mu)$ is constructed so that $\{\Xi_{P,m,N}(\xi_n,\mu) \,:\, n\in B\}$ can be distinguished from $\{\Xi_{P,m,N}(\xi_n,\mu) \,:\, n\in B^c\}$ by algorithmic means provided that $\Xi_{P,m,N}(\cdot,\cdot,\mu)$ is Borel-Turing computable. However, one can construct a Turing machine that decides $n\in B$ or $n\in B^c$ for arbitrary $n\in \N$, contradicting the non-recursiveness of $B$. Subsequently, the construction can be verified for the considered functions $\Xi_{\text{BP},m,N}(\cdot,\cdot,\varepsilon)$ and $\Xi_{\text{Lasso},m,N}(\cdot,\cdot,\lambda)$.  
\end{proof}
\begin{Remark}\label{rm:TM_AA}
    The statement in the theorem can be strengthen by providing a lower bound on the achievable algorithmic approximability of the problem, i.e., how precise single-valued restrictions of the sought reconstruction map can be approximated by Borel-Turing computable functions. Note that the limitations do not arise due to the unboundedness of the input domain, but hold true on a compact input set. Furthermore, algorithmic non-solvability is not connected to poor conditioning of the inverse problem instances; one can construct input domains with well-conditioned instances that suffer from the same limitations. For details we refer to \cite{Boche2022LimitsDL}. 
\end{Remark}
By invoking \Cref{thm:EquivTransCompTuring}, we infer that no transparent algorithms to solve inverse problems exist on Turing machines.
\begin{Corollary}
    In the setting of \Cref{thm:nonApprox}, there does not exist a transparent algorithm solving inverse problems described by $\Xi_{\text{BP},m,N}(\cdot,\cdot,\varepsilon)$ and $\Xi_{\text{Lasso},m,N}(\cdot,\cdot,\lambda)$.
\end{Corollary}

\subsection{Algorithmic Solvability of Inverse Problems in BSS Model}
In the BSS setting, a general algorithmic non-solvability statement does not hold. We indeed can establish algorithmic solvability under specific circumstances. However, to do so we need to distinguish between a real and complex domain since BSS machines show distinct behaviour depending on the underlying structure. 

\subsubsection{Real Case}
First, we consider the real case. Given a multi-valued mapping $f:\mathcal{V} \rightrightarrows \mathcal{Z}$, we denote by $f^\R$ its restriction to real inputs and outputs. Although only the complex domain was studied explicitly in \Cref{thm:nonApprox}, the algorithmic non-solvability in the Turing model remains valid for basis pursuit \eqref{eq:sparseprob} described by the mapping $\mathcal{M}_{\Xi^\R_{\text{BP},m,N}(\cdot,\cdot,\varepsilon)}$ since the proof idea translates to the strictly real case. In contrast, the same problem is algorithmically solvable in the BSS model.   
\begin{Theorem}[\cite{Boche2022InvProb}]\label{thm:BSSReal}
    Consider the optimization problem \eqref{eq:sparseprob} restricted to the real domain and the associated mapping $\Xi^\R_{\text{BP},m,N}$. The problem described by $\Xi^\R_{\text{BP},m,N}$ is algorithmically solvable on BSS machines.
\end{Theorem}
\begin{proof}[Proof Sketch]
     The approach is to rewrite the problem such that an established algorithm for finding a minimizer of a polynomial on a semialgebraic set can be applied. Due to the restriction to the real domain, the involved terms can indeed be transferred to the required setting.
\end{proof}
\begin{Remark}
    Note that the optimization parameter acts as an additional input to the mapping $\Xi^\R_{\text{BP},m,N}$ (whereas in the Turing setting in \Cref{thm:nonApprox} the optimization parameter was fixed beforehand). Therefore, in the BSS setting we proved the existence of an even stronger algorithm (with an additional input parameter) than the one assessed in \Cref{thm:nonApprox}. For more details and the proof we refer to \cite{Boche2022InvProb}.      
\end{Remark}
\begin{Remark}
    Despite the existence of an algorithm solving the problem $\Xi^\R_{\text{BP},m,N}$, its computational complexity may remain inappropriately high. Hence, the theorem only provides a theoretical existence result neglecting the question of practical implementation.
\end{Remark}
In case of (square root) lasso optimization \eqref{eq:lasso}, algorithmic solvability can not be established on BSS machines, not even on the restricted real domain. The underlying issue is the BSS non-computability of the square root function on the real numbers, which arises due to the algebraic structure of the BSS model \cite{Blum98ComplRealComp}. Hence, the $\ell_2$ norm is not BSS computable, which renders (square root) lasso optimization infeasible on BSS machines. Note that the explicit computation of the $\ell_2$ norm can be avoided for basis pursuit, because it arises only in the description of the constraint, whereas for (square root) lasso it is directly incorporated in the objective. Once can circumvent this problem by assuming that BSS machines possess an additional module that can be called to compute the square root. This is motivated by the fact that the square root is an elementary function that should be computable in a practical model as is the case for Turing machines \cite{Pour-El17Computability}. Without this additional assumption, we either need to modify or approximate the objective of the corresponding optimization problem to obtain algorithmic solvability on BSS machines. We will consider the former approach and return to the latter approach in the complex setting. Instead of square root lasso, we can consider lasso optimization \cite{Belloni11SquareRootLasso, Tropp06Relax, Lv2011GroupLasso} given by 
\begin{equation}\label{eq:lassoSq}
    \argmin_{x \in \C^N} \lambda \norm[\ell_1]{x} + \norm[\ell_2]{Ax -y}^2. \tag{$\text{Lasso}^2$}
\end{equation}
Here, the objective does not require the computation of the square root function (due to the squaring of the $\ell_2$ norm) and, indeed, this change is sufficient to establish algorithmic solvability via the same approach as in \Cref{thm:BSSReal}.
\begin{Theorem}[\cite{Boche2022InvProb}]\label{thm:BSSreal2}
    Consider the optimization problem \eqref{eq:lassoSq} restricted to the real domain and the associated mapping $\Xi^\R_{\text{Lasso}^2,m,N}$. The problem described by $\Xi^\R_{\text{Lasso}^2,m,N}$ is algorithmically solvable on BSS machines.
\end{Theorem}
\begin{Remark} 
    The analogous proof technique as for the square root lasso minimization problem in \Cref{thm:nonApprox} can be applied to derive algorithmic non-solvability of lasso minimization \eqref{eq:lassoSq} on Turing machines \cite{Boche2022InvProb}.
\end{Remark}

Applying \Cref{thm:EquivTransCompBSS}, we conclude that a transparent algorithm to solve real inverse problems does exist on BSS machines.
\begin{Corollary}
    In the setting of \Cref{thm:BSSReal} and \Cref{thm:BSSreal2}, there does exist a transparent algorithm solving inverse problems described by $\Xi^\R_{\text{BP},m,N}$ and $\Xi^\R_{\text{Lasso}^2,m,N}$, respectively.
\end{Corollary}

\subsubsection{Complex Case}
For BSS machines operating on the complex numbers, we have to choose a suitable representation. As described in \Cref{rm:ComplexRepresentation}, considering complex numbers as entities results in the non-computability of elementary complex functions. In contrast, identifying $\C$ with $\R^2$ and representing complex inputs $z$ in form of $(\Re(z), \Im(z))$ circumvents this problem to a certain degree. However, even in the $\R^2$-representation $\ell_p$ norms are in general not computable functions, since they require the computation of a square root (which is not a real BSS computable function). Similarly to the real case, we can introduce adjusted optimization problems that promote solutions with similar properties as the original problems but allow for algorithmic solvability in the BSS model. For details, we refer to \cite{Boche2022InvProb} and wish to mention that algorithmic solvability in the Turing model is still not achievable for the adjusted problems. Although the adaptation of the objectives try to maintain the original structural properties, they are not derived by rigorous reasoning.   

Instead of replacing the optimization problem, we can also approximate its objective. This approach is demonstrated for basis pursuit by establishing an adequate (BSS computable) approximation of the $\ell_1$ norm.
\begin{Theorem}[\cite{Boche2022InvProb}]\label{thm:BPapprox}
    Let $\beta, \gamma > 0$. For $A\in\C^{m\times N}$, $y\in \C^m$ and $\varepsilon >0$ consider the optimization problem 
    \begin{equation*}\label{eq:approxbp}
        \argmin_{x \in I_\beta} p_{\beta, \gamma}(x)  \text{ such that }   \norm[\ell_2]{Ax -y} \leq \varepsilon, 
        \tag{BP-A}
    \end{equation*}
    where $p_{\beta, \gamma}$ is a polynomial satisfying 
    \begin{equation*}\label{eq:approxl1}
        \sup_{x \in I_\beta}\abs{\norm[\ell_1]{x} - p_{\beta, \gamma}(x)} \leq \gamma
    \end{equation*}
    and 
    \begin{equation*}
        I_\beta \coloneqq \{x \in \C^N: \norm[\ell_2]{x} < \sqrt{N} \beta\}    
    \end{equation*}
    Then, the problem described by $\Xi_{\text{BP-A},m,N}$ is algorithmically solvable in the BSS model.
\end{Theorem}
\begin{proof}[Proof Sketch]
    Applying the Weierstrass approximation theorem, in particular its constructive proof via Bernstein polynomials, we can derive a (BSS computable) polynomial $p_{\beta, \gamma}$ approximating the $\ell_1$ norm up to an error of $\gamma$ on $I_\beta$. Finally, the algorithmic solvability of $\Xi_{\text{BP-A},m,N}$ follows along the same lines as in the proof of \Cref{thm:BSSReal}. 
\end{proof}
\begin{Remark}
    The objective $p_{\beta, \gamma}$ in \eqref{eq:approxbp} approximates up to an error of $\gamma$ the $\ell_1$ norm, i.e., the objective of basis pursuit optimization, on the set $I_\beta$. In this sense, \eqref{eq:approxbp} represents an approximation of basis pursuit if its minimizers are contained in $I_\beta$. Additionally, one can construct a BSS computable function that decides for input $(A,y,\varepsilon)$ if basis pursuit \eqref{eq:sparseprob} has at least one solution and if the solution(s) are contained in $I_\beta$. Therefore, there does exist a BSS machine $\mathcal{B}_{\beta,\gamma}$ that checks if the solutions of basis pursuit for $(A,y,\varepsilon)$ are contained in $I_\beta$. If the answer is positive, a solution of \eqref{eq:approxbp} is computed consecutively. Otherwise the computation is aborted since the approximation accuracy $\gamma$ and acceptance domain depending on $\beta$ can not be adjusted autonomously. In other words, for each pair of parameters $(\beta, \gamma)$ a distinct BSS machine $\mathcal{B}_{\beta,\gamma}$ needs to be constructed. Moreover, note that the obtained minimizers of \eqref{eq:approxbp} and basis pursuit need not to agree and we do not obtain worst-case bounds on their distance, for details we refer to \cite{Boche2022InvProb}.  
\end{Remark}
\begin{Remark}
    In the Turing model, the outlined approach to approximate basis pursuit is not feasible \cite{Boche2022InvProb}. Even more, due to lower bounds on the algorithmic approximability in the Turing model (see \Cref{rm:TM_AA}), different approximation schemes necessarily have certain limits in this setting.
\end{Remark}
Finally, via \Cref{thm:EquivTransCompBSS} we can state a similar result with respect to trustworthiness as in the real case.
\begin{Corollary}
    In the setting of \Cref{thm:BPapprox}, there does exist a transparent algorithm approximating inverse problems described by $\Xi_{\text{BP-A},m,N}$.
\end{Corollary}

\subsection{Comparison of Results in Turing and BSS model}

The presented findings indicate that the degree of algorithmic solvability of inverse problems depends on both the considered problem description and computing model. In the Turing model, there exists a rather general algorithmic non-solvability statement, also supported by the results in \cite{colbrook21stable}, whereas the landscape is more diverse in the BSS setting. Here, the potential of algorithmic solvability is connected to the specific properties of the underlying optimization problem, which can to a certain extent be positively influenced by modifying or approximating the objective. Although the adjustments may typically not be applied in practice, they maintain the properties of the original formulation to some degree and show that a wide range of inverse problem description is in principle algorithmically solvable in the BSS model.  

In contrast, related approaches appear to not be feasible in the Turing model. On the one hand, (reasonable) modifications of the objectives of the optimization problems do not influence algorithmic solvability, since algorithmic non-solvability is related to properties of the underlying solution set of the given problem, which pertain a broad class of inverse problem descriptions. On the other hand, algorithmic non-approximability in the Turing model also renders approximate approaches impractical. Characterizing classes of inputs that allow for algorithmic solvability and thereby identifying problematic inputs, which violate performance guarantees, could potentially alleviate the non-computability issue. However, it was found that implementing an exit-flag functionality, i.e., aborting the computation and notifying the user once a `problematic' input is recognized, on Turing machines is not feasible for inverse problems in general \cite{bastounis21extended}.  

Consequently, we can observe a gap in algorithmic solvability between the Turing and the BSS model. In particular, BSS machines provide a greater capacity to solve inverse problems algorithmically. However, the  
power of the BSS model is strictly connected to the (exact) representation and processing of real numbers as entities. If the same approximating sequences as in the Turing model are used to represent real numbers, then essentially the same limitations as in the Turing model arise in the BSS model as well \cite{colbrook21stable}. It should also be noted that algorithmic solvability of inverse problems was assessed in a very general framework, i.e., we did not consider a specific but a broad class of inverse problems. Hence, by restricting to a more narrow framework consisting of a limited number of problems the degree of algorithmic solvability may change. Thereby, specific properties of the considered problems could be exploited, which at the same time can not be incorporated in a more universal approach. Thus, a trade-off between generality and trustworthiness expressed through AgT may not be avoidable in our framework, but the degree may vary with the underlying computing paradigm.

\section*{Acknowledgements}
This work of H. Boche was supported in part by the German Federal Ministry of Education and Research (BMBF) within the national initiative on 6G Communication Systems through the Research Hub 6G-life under Grant 16KISK002.

This work of Gitta Kutyniok was supported in part by the Konrad Zuse School of Excellence in Reliable AI (DAAD), the Munich Center for Machine Learning (BMBF) as well as the German Research Foundation under Grants DFG-SPP-2298, KU 1446/31-1 and KU 1446/32-1. Furthermore, G. Kutyniok acknowledges support from LMUexcellent, funded by the Federal Ministry of Education and Research (BMBF) and the Free State of Bavaria under the Excellence Strategy of the Federal Government and the Länder as well as by the Hightech Agenda Bavaria.





\printbibliography

\end{document}